\documentclass[a4paper,fleqn]{cas-sc}

\usepackage[authoryear,longnamesfirst]{natbib}

\usepackage{graphicx}
\usepackage{amsmath}
\usepackage{amsthm}

\usepackage{booktabs}
\usepackage{algorithm}
\usepackage{algorithmic}

\usepackage{multirow}
\usepackage{rotating}

\newtheorem{theorem}{Theorem}
\newtheorem{example}{Example}
\def\tsc#1{\csdef{#1}{\textsc{\lowercase{#1}}\xspace}}
\tsc{WGM}
\tsc{QE}
\tsc{EP}
\tsc{PMS}
\tsc{BEC}
\tsc{DE}

\begin{document}
\let\WriteBookmarks\relax
\def\floatpagepagefraction{1}
\def\textpagefraction{.001}
\shorttitle{Overestimation Bias: Settling
Large Discrete Action Space via Action Intersection}

\title [mode = title]{Revisiting Overestimation Bias Problem of Q-learning: Settling Large Discrete Action Space via 
Action Intersection}                      

\author[1]{Pu Li}
\affiliation[1]{
organization={Chongqing Institute of Green and Intelligent Technology},
state={Chongqing},
country={China},
}

\author[2]{Tao Tan}
\affiliation[2]{
organization={University of Science and Technology of China},
city={Hefei},
state={Anhui},
country={China},
}
\author[2]{Hong Xie}

\author[1]{Xiaoyu Shi}

\author[1]{Mingsheng Shang}

\begin{abstract}
This paper considers the overestimation bias problem of Q-learning in the setting of a large action space, for the purpose of relieving the bottleneck of existing methods. 
We find that the large action space increases the randomness in Q-value estimation. The randomness makes two paradigms that drive the major literature on the overestimation problem have their own bottlenecks: 
the coupling paradigm, i.e., the optimal action and its Q-value are estimated with the same Q-function, always has a positive bias. This is because randomness leads to some actions having abnormally high estimated values than their true values, and the coupling methods prefer these actions. The decoupling paradigm, i.e., the optimal action and its Q-value are estimated with two independent Q-functions, always has a negative bias. This is because randomness increases the estimation gap between the two independent Q-tables for the same action.
This paper shows that action intersection can be a simple yet powerful strategy to relieve these bottlenecks.
The action intersection strategy enables semi-decoupling via two designs: 
(1) it allows two Q-functions to share a certain fraction of trajectory data; 
(2) if a data sample is shared, each Q-function is updated using the coupling paradigm; otherwise, using the decoupling paradigm.  
Two properties make the action intersection strategy powerful:  
(1) attaining a large bias range, i.e., varying the data sharing fraction, the estimation bias varies from underestimating to overestimating; 
(2) fine granularity: the action intersection size can be made arbitrarily finer to enable finer control. 
We consider two experiment settings, i.e., tabular and deep RL, 
deep RL experiments show that our method outperforms several SOTA baselines drastically; 
tabular experiments reveal why our method can achieve superior performance.  
We believe that our work opens a new door for mitigating the overestimation bias of Q-learning. 
\end{abstract}



\maketitle

\section{Introduction}
Q-learning \cite{watkins1992q} is one of the most widely adopted and studied Reinforcement Learning (RL) algorithms \cite{chu2022formal}\cite{zhang2023m2dqn}\cite{zhao2023inverse}. 
However, Q-learning suffers from the overestimation bias\cite{thrun2014issues}\cite{hasselt2010double}. 
Previous work shows that in Q-learning, function approximators induce noise in the output predictions, and such noise can lead to a systematic overestimation of Q-values. 
Note that the noisy estimates of Q-values 
are caused by stochastic and unknown rewards and state transitions.  
Double Q-learning is a classical algorithm to remove the overestimation bias of Q-learning,
but it may lead to the underestimation bias.
Due to the simplicity and convergence guarantees under some common assumptions \cite{melo2001convergence},
these value-estimation algorithms, such as Q-learning and Double Q-learning, 
have become a practical tool, for instance, 
as a value estimator in RL fine-tuning Large Language Models\cite{zhai2024fine}\cite{rafailov2024r},
where the action space is large. 
Considering that other potential application scenarios also feature large action spaces (e.g., recommender systems), we ask the following question:\textbf{\textit{Will large action spaces become a bottleneck for bias control?}}

To answer this question, we conduct an experiment to investigate how the estimation bias changes as the number of arms increases. The result shows that indeed large action space amplifies the estimation bias. Example \ref{example:estimationBias} illustrates our empirical study.
\begin{example}
Consider a multi-armed bandit setting,
which includes one state $S$ with $N$ actions $\{a_1,...,a_{N}\}$.
Following previous work \cite{zhang2017weighted},
each action $a_i$ returns a reward obeys $\mathcal{N}(0, 10^2)$.
We set the discount factor $\gamma=0.95$,
and the maximum expected Q-value is 0. Our experiments investigate how the estimation error of these algorithms varies as the action space size increases.
Figure \ref{example:estimationBias}(A) shows the maximum Q-value of Q-learning algorithm after a 10,000-step update with different numbers of arms. We can observe that when the number of arms is 80, the curve of Q-learning is at the top, and when the number of arms is 20, the curve is at the bottom. And all curves of Q-learning are above the curve of unbiased curve. This implies that Q-learning has a positive estimation bias, and the estimation bias increases as the size of the action space increases.
Figure \ref{example:estimationBias}(B) shows the maximum Q-value of the Double Q-learning algorithm after a 10,000-step update with different numbers of arms. We can observe that when the number of arms is 80, the curve of Double Q-learning is at the bottom, and when the number of arms is 20, the curve is at the top. And all curves of Double Q-learning are under the curve of the unbiased curve. This implies that Double Q-learning has a negative estimation bias, and the estimation bias increases as the size of the action space increases.
Figure \ref{example:estimationBias}(C) shows the relationship between the size of the action space and the maximum Q-values.
One can observe that as the action space increases, the length of all bars rises, i.e., the maximum Q-value estimated by each method progressively deviates from the optimal Q-value of 0. This implies that current methods suffer from estimation bias in environments with large action spaces, and the estimation bias and the size of the action space have a positive correlation.
More experiment details refer to Section \ref{sec:exp-bandit}.
\begin{figure}
  \centering
  \includegraphics[width=\linewidth]{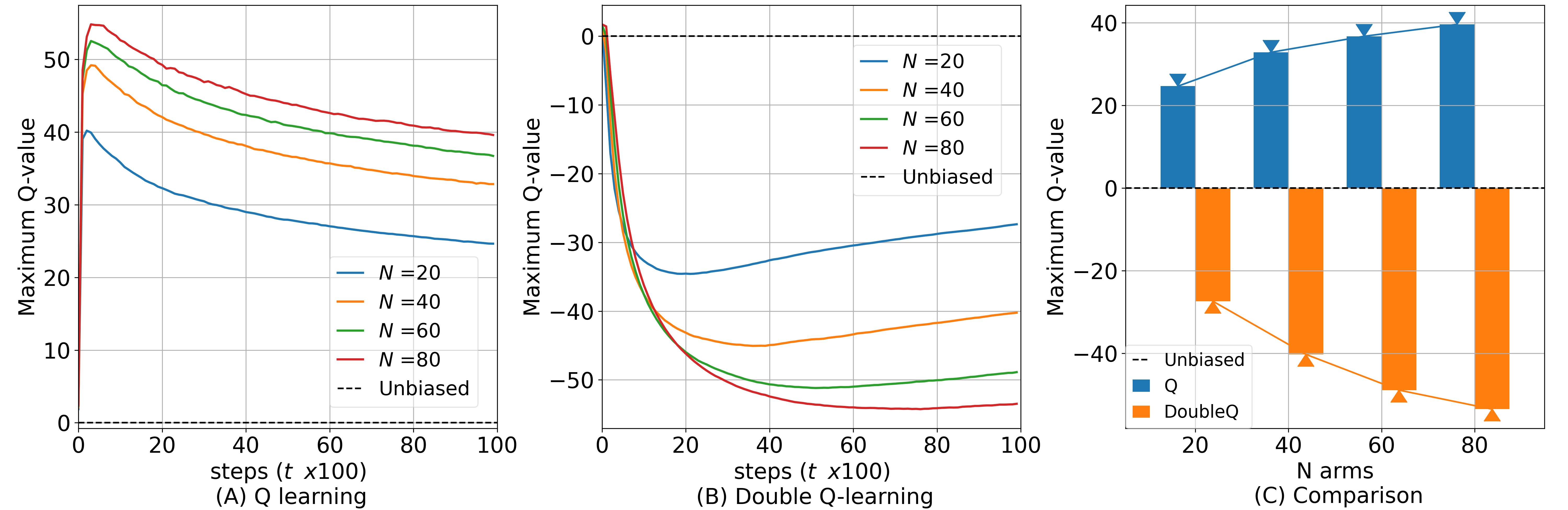}
  \caption{The maximum Q-value of Q-learning and Double Q-learning in a large action space.}
  \label{fig:sample}
\end{figure}

\label{example:estimationBias}
\end{example}

Existing estimation bias control methods also face the challenge of the large action space. To analyze the reason, we categorize existing methods into two classes, coupling methods and decoupling methods, for separate discussion. 
Coupling methods evaluate the maximum expected Q-value using a single Q-function, i.e., both the optimal action and its Q-value are obtained from the same Q-function, and always have a positive bias. Q-learning, Maxmin Q-learning, Averaged Q-learning, and Order Q-learning all belong to coupling methods. More specifically, in the setting of a large action space, the randomness in Q-values increases with the number of actions, leading to some actions having abnormally higher estimated values than their true values. Coupling-based methods consistently select these overestimated actions to update the estimator, thereby introducing positive bias.
Decoupling methods evaluate the maximum expected Q-value using two independent Q-functions, i.e., the optimal action and its Q-value are estimated by two independent Q-functions, and always have a negative bias. Double Q-learning and EBQL belong to decoupling methods. In this setting of a large action space, the estimation gap between the two independent Q-tables for the same action widens as the number of actions increases. The estimated value given by Q-table 2 for the optimal action selected by Q-table 1 may be abnormally low relative to the true value, thus resulting in negative bias.

According to the above analysis, the core challenge in the setting of a large action space is the high randomness in selecting the optimal action and estimating its value. To alleviate this problem, we propose an action intersection strategy based on Double Q-learning, and obtain \underline{A}ction \underline{I}ntersection \underline{D}ouble \underline{Q}-learning (AIDQ). The core idea is maintaining two Q-tables (e.g., $Q^1$ and $Q^2$), and using $Q^1$  to select a set of optimal actions(e.g., $topK$ actions), using $Q^2$ to identify the optimal action in the set by estimating their Q-values. Note that selecting $topK$ actions is a simple yet effective strategy in the large action space setting, because it decreases the randomness in selecting the optimal action and estimating its value. 
When $K=1$, our method turns to Double Q-learning, which has a negative bias. 
When $K=|\mathcal{A}|$, where $|\mathcal{A}|$ is the size of the action space, our method turns to two Q-tables alternately performing Q-learning updates, which has a positive bias. This feature helps to enable fine-grained control over a large bias range, from negative bias to positive bias in the large action space setting.

In summary, this paper makes the following contributions:
\begin{itemize}
\item To control the overestimation bias in a large action space setting, i.e., 
the estimation bias is always positive or negative,
this paper proposed an action intersection strategy,
i.e., control the estimation bias to vary continuously from negative bias to positive bias.
\item Based on the Double Q-learning, we introduce an action intersection Double Q-learning,
which is a simple yet powerful method that enables fine-grained control over a large bias range in the setting of a large action space.
\item We evaluate our method in both tabular and deep RL settings.
Extensive experiment results show that our
method outperforms SOTA baselines drastically in different settings. 
More specifically,
our method improves the average return per episode over SOTA baselines by at least 53.21\%
in the Asterix environment.
\end{itemize}

\section{Related Work}

The overestimation problem of Q learning was first identified and researched by Sebastian Thrun et al. in \cite{thrun2014issues}. Specifically, this work shows that function approximators induce noise in the output predictions, and such noise can lead to a systematic overestimation of Q-values. As a result of this overestimation bias, the agent could fail to learn an optimal policy in certain situations. To address the overestimation problem, many methods were proposed. We divide works aiming to address overestimation bias in Q learning into two categories: coupling method and decoupling method.

\subsection{Coupling Method}
They evaluate the maximum expected Q-value using a single Q-function,
i.e., both the optimal action and its Q-value are obtained from the same Q-function,
always has a positive bias.
Double Q-learning\cite{hasselt2010double} points out that the max operator leads to the significant issue of overestimation. Motivated by this perspective, Softmax Q-learning \cite{song2019revisiting}, Soft Mellowmax Q-learning \cite{song2019revisiting} 
and Mellowmax Q-learning \cite{asadi2017alternative} use the softmax operator to replace the max operation in the generalized value iteration framework. And they control the estimation bias via the softmax temperature hyperparameter. Their theoretical results quantify how the softmax Bellman operator reduces the overestimation bias.
Also motivated by the perspective that the max operator in Q-learning can lead to overestimation, Averaged Q-learning \cite{anschel2017averaged} suggests a different solution to the overestimation phenomenon. Averaged Q-learning points out that the magnitude of the overestimation bias is controlled by the target approximation error variance. Thus, Averaged Q-learning builds an averaged Q-estimator over multiple Q-functions to reduce the target approximation error variance and uses this averaged Q-estimator to estimate the maximum expected Q-value. The estimation bias of Averaged Q-learning is inversely proportional to the number of Q-functions. However, with a finite number of Q-functions, it always has a positive bias.
Maxmin Q-learning \cite{lan2020maxmin} builds a minimum Q-estimator over multiple Q-functions, and uses this minimum Q-estimator to estimate the maximum expected Q-value. The estimation bias of Maxmin Q-learning is inversely proportional to the number of Q-functions. However, the bias control is coarse-grained and discrete, i.e., it depends on the number of Q-functions.
AdaEQ \cite{wang2021adaptive} is a natural extension of Maxmin Q-learning to adaptively select the number of Q-functions
and aims to remove the estimation bias. The theoretical finding shows that approximation error characterization can serve as the feedback for flexibly controlling the ensemble size.
Doubly Bounded Q-Learning \cite{ren2021estimation} leverages an abstracted dynamic programming as a second value estimator to reduce the underestimation bias and to eliminate non-optimal fixed points of Double Q-learning.
Order Q-learning \cite{tan2024oder}
builds an order Q-estimator over multiple Q-functions,
and uses this order Q-estimator to estimate the maximum expected Q-value.
However, the bias control is discrete, i.e., it depends on the number of Q-functions and the index of the order statistic function.

\subsection{Decoupling Methods}
They evaluate the maximum expected Q-value using two independent Q-functions,
i.e., the optimal action and its Q-value are estimated by two independent Q-functions,
always has a negative bias.
Double Q-learning \cite{hasselt2010double} uses two independent Q-functions to evaluate the maximum expected Q-value,
where one is used to select the optimal action and the other to estimate its Q-value.
It removes the overestimation bias of Q-learning, but leads to the underestimation bias.
Note that this underestimation bias can cause a slower learning speed and a larger performance penalty than the overestimation bias \cite{abliz2022underestimation}.
EBQL \cite{peer2021ensemble} is a natural extension of Double Q-learning to ensembles; this works aim to address the underestimation bias of Double Q-learning. It maintains $K$ Q-functions. For each update step, it randomly chooses a Q function to select the optimal action and uses all the other functions to estimate the Q-value of the optimal action. Theoretical analysis proves that EBQL-like updates yield lower MSE when estimating the maximal mean of a set of independent random variables. However, it always has a negative bias.
AC-CDQ \cite{jiang2021action} attempts to balance the overestimation bias and the underestimation bias via a clipping operation. However, the bias control is coarse-grained and discrete, i.e., it depends on the number of action candidates.

There are also some other works on the setting of a large action space.
\cite{dulac2015deep} proposed approach leverages prior information about the actions to embed them in a continuous space upon which it can generalize. And then uses an approximate nearest neighbor search to find the set of closest discrete actions in logarithmic time.
\cite{weisz2018sample} examines the application of reinforcement learning to dialogue policy optimisation in a spoken dialogue system, which has a relatively large action set. The key algorithm is ACER\cite{wang2016sample}\cite{munos2016safe}, which combines a lot of technologies such as actor-critic methods, off-policy reinforcement learning with experience replay, and various methods aimed at reducing the bias and variance of estimators. Although this algorithm has achieved relatively good performance, it is overly complex and lacks theoretical analysis.
\cite{majeed2021exact} addresses the large action-space problem by sequentializing actions, which can reduce the action space size significantly, even down to two actions. However, it leads to an increased planning horizon, thereby increasing costs.
\cite{ming2023cooperative} proposes to decompose the complex decision task in a large action space into multiple decision sub-tasks in small action subsets, using a rule-based action division method.
However, these methods are not specifically designed to address the overestimation problem in environments with large action spaces. Additionally, these methods primarily focus on the deep learning domain, often involving complex algorithms and models, and offer limited theoretical analysis of estimation errors.

This paper proposes a new semi-decoupling method,
which evaluates the maximum expected Q-value using two dependent Q-functions,
i.e., the optimal action and its Q-value are estimated by two dependent Q-functions.
We build the dependence between two Q-functions via an action intersection strategy,
and enable fine-grained and continuous control over a large bias range, i.e., 
varying the data sharing fraction smoothly shifts the bias from negative to positive.

\section{Preliminary}
We consider a Markov Decision Process (MDP \cite{chen2021randomized}) 
defined by the tuple $(\mathcal{S}, \mathcal{A}, P, R, \gamma)$,
where $\mathcal{S}$ is the state space,
$\mathcal{A}$ is the action space,
$P(s' \mid s, a)$ is the state transition probability,
$R(s,a)$ is the reward function,
and $\gamma \in (0,1)$ is the discount factor.
At each time step $t \in \mathbb{N}$,
the agent observes the current state $s_t \in \mathcal{S}$,
selects an action $a_t \sim \pi(\cdot \mid s_t)$ according to a policy $\pi$,
transitions to the next state $s_{t+1} \sim P(\cdot \mid s_t, a_t)$,
and receives a reward $r_t = R(s_t, a_t)$.
The objective of the agent is to learn an optimal policy that maximizes the expected discounted return.

\subsection{Q-learning}
Q-learning \cite{watkins1992q} maintains a single Q-table denoted by $Q_t(s,a)$.
It first initializes the Q-table $Q_{0}(s,a)$,
and then collects the interaction sample \((s_t,a_t,r_t,s_{t+1})\) by $\epsilon$-greedy policy \cite{ghasemipour2021emaq}, 
where $\epsilon\in[0,1]$.
Note that under the $\epsilon$-greedy policy,
the agent selects a greedy action
from $\arg\max_{a \in \mathcal{A}} Q_t(s_t, a)$ with probability $(1-\epsilon)$,
and selects a random action from $\mathcal{A}$ with probability $\epsilon$.
In each time step $t$,
the updating rule of Q-learning is:
\begin{equation}
    \begin{cases}
        a^*_{t+1} =  \arg\max_{a_{t+1}} Q_t(s_{t+1}, a_{t+1}), \\
        Y_t = r_t + \gamma Q_{t}(s_{t+1}, a^*_{t+1}), \\
        Q_{t+1}(s_t,  a_t)  = (1-\alpha) Q_{t}(s_t, a_t) + \alpha Y_t,
    \end{cases}  
\label{equ:q}
\end{equation}
where \(\alpha\) is the learning rate,
$Y_t$ is the TD target Q-value for $Q_t(s_t,a_t)$.

Equation (\ref{equ:q}) illustrates that, during the Q-learning update,
both the optimal action selection and its value evaluation rely on the same Q-table.
This coupling mechanism is the root cause of the overestimation bias in Q-learning (i.e., $\mathbb{E}[\max_a Q(s,a)] \geq \max_a\mathbb{E}[Q(s,a)]$),
which can be theoretically explained in the Lemma of \cite{thrun2014issues}.

\subsection{Double Q-Learning}

Double Q-learning \cite{hasselt2010double} maintains two independent Q-tables denoted by $Q^{A}_t(s,a),Q^{B}_t(s,a)$.
Different from Q-learning,
Double Q-learning uses the $\epsilon$-greedy policy with 
$\arg\max_{a \in \mathcal{A}} \frac{Q^{A}_t(s_t, a)+Q^{B}_{t}(s_t,a)}{2}$ to sampling.
In each time step $t$,
Double Q-learning random selects a Q-table, such as $Q^{A}_{t}$,
to update as:
\begin{equation}
    \begin{cases}
        a^*_{t+1} = \arg\max_{a_{t+1}} Q_t^A(s_{t+1}, a_{t+1}), \\
        Y_t^A = r_t + \gamma Q_{t}^B(s_{t+1}, a^*_{t+1}), \\
        Q_{t+1}^A(s_t,  a_t) = (1-\alpha) Q_{t}^A(s_t, a_t) {+} \alpha Y_t^A,
    \end{cases} 
\label{equ:dq} 
\end{equation}
where $Y_t^A$ is the TD target Q-value for $Q_t^A(s_t,a_t)$.

Equation (\ref{equ:dq}) illustrates that, during the Double Q-learning update,
the optimal action selection and its value evaluation rely on two independent Q-tables.
This decoupling mechanism is the root cause of the underestimation bias in Double Q-learning, 
which can be theoretically explained in Lemma 1 of \cite{hasselt2010double}.

\section{Method}
Example \ref{example:estimationBias} shows that existing methods fail in settings with large action spaces. To overcome this problem, this paper proposed a new method called \underline{A}ction \underline{I}ntersection \underline{D}ouble \underline{Q}-learning (AIDQ). The core idea is maintaining two Q-tables (e.g., $Q^1$ and $Q^2$), and using $Q^1$ to select a set of optimal actions (e.g., $topK$ actions), using $Q^1$ to identify the optimal action in the set by estimating their Q-values.

\begin{figure}
  \centering

  \includegraphics[width=\linewidth]{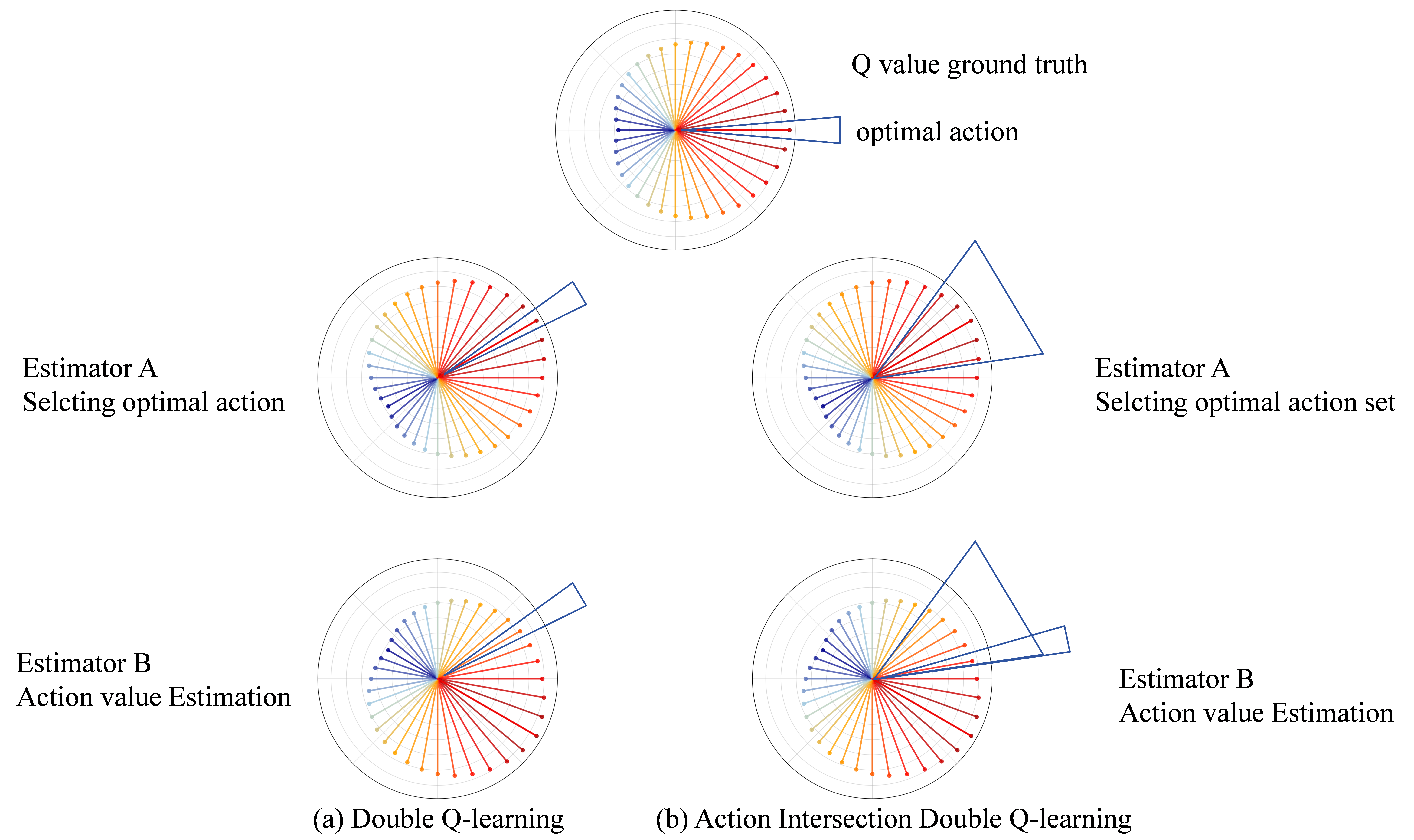}
  \caption{Double Q learning with action intersection}
  \label{fig:method}
\end{figure}

Figure \ref{fig:method}(b) shows our action intersection Double Q-learning.
We denote the two Q-functions at time step $t$ as $Q^1_t$ and $Q^2_t$. Firstly, randomly initialize two Q-functions as $Q^1_0$ and $Q^2_0$ at time step $t=0$.
Secondly, at any time step $t$, we generate a policy $\pi_t:\mathcal{S} \rightarrow \mathcal{A}$ by  
\begin{equation}
    a=\pi_t(s) \stackrel{\mathrm{def}}{=} \arg\max_{a'}\left[Q_t^1(s,a')+Q_t^2(s,a')\right].
\end{equation}

Secondly, at any time step $t$, the agent is at state $s_t$ and uses $\epsilon$-greedy policy with $\pi_t$ to decide the action $a_t$ to take where $\epsilon\in[0,1]$:
\begin{equation}
    a_t =
    \begin{cases}
        Sample(\mathcal{A}) &Prob.=\epsilon,\\
        \pi_t(s_t) &Prob.= 1-\epsilon,
    \end{cases}
\end{equation}
where $sample(\mathcal{A})$ is to choose a random action from $\mathcal{A}$ with equal probability. The environment receives the action $a_t$ and then gives the reward $r_t$ and the next state $s_{t+1}$.
This process generates a sampling data $(s_t, a_t,r_t, s_{t+1})$. The data will be used to update the Q-functions. 

Thirdly, randomly select either $Q^1_t$ or $Q^2_t$ with equal probability to update (e.g., select $Q^1_t$ to update). Then we select the $topK$ actions with the highest Q-values for state $s_{t+1}$ according to $Q^1_t$ as the action intersection set $\mathcal{A}_{t+1}^{K*}$:
\begin{equation}
    \mathcal{A}_{t+1}^{K*}= \operatorname{arg\,topK} \limits_{a' \in \mathcal{A}}Q^1_t(s_{t+1}, a'),
\end{equation}
where we extend the max operation to a topK operation. This set is called the action intersection set because it is not only selected by $Q^1_t$, but also used by $Q^2_t$. 
Note that $|\mathcal{A}_{t+1}^{K*}|=topK$.
The other Q-function $Q^2_t$ selects the optimal action for state $s_{t+1}$ from the action intersection set $\mathcal{A}_{t+1}^{K*}$ and gives its Q-value as:
\begin{equation}
    \max_{a_{t+1} \in \mathcal{A}_{t+1}^{K*}}Q^2_t(s_{t+1}, a_{t+1}).
\end{equation}

The target value $Y_t(s_t, a_t)$ for $Q^1_t$ is calculated as:

\begin{equation}
    Y^1_t(s_t, a_t) = r_t + \gamma \max_{a_{t+1} \in \mathcal{A}_{t+1}^{K*}}Q^2_t(s_{t+1}, a_{t+1}).
\end{equation}

Finally, update $Q^1_t$ as:
\begin{equation}
    Q^1_{t+1}(s_t, a_t) = (1-\alpha)Q^1_{t}(s_t, a_t) + \alpha Y^1_t(s_t, a_t).
\end{equation}

If $Q_t^2$ is chosen to update, the update rule for $Q_t^2$ is defined by swapping the roles of the two Q-functions in the update rule of $Q_t^1$, which is given by:
\begin{equation}
    \begin{cases}
        \mathcal{A}_{t+1}^{K*}= \operatorname{arg\,topK} \limits_{a' \in \mathcal{A}}Q^2_t(s_{t+1}, a'), \\
        Y_t^2(s_t, a_t) = r_t + \gamma \max_{a_{t+1} \in \mathcal{A}_{t+1}^K*}Q^1_t(s_{t+1}, a_{t+1}), \\
        Q_{t+1}^2(s_t, a_t) = (1-\alpha)Q_t^2(s_t, a_t) + \alpha Y_t^2(s_t, a_t).
    \end{cases}  
\end{equation}

Algorithm \ref{alg:ActionIntersectionDouleQlearning} outlines our \underline{A}ction \underline{I}ntersection \underline{D}ouble \underline{Q}-learning (AIDQ).
When $topK=1$, our method turns to Double Q-learning, which has a negative bias. 
When $topK=|\mathcal{A}|$, where $|\mathcal{A}|$ is the size of the action space, our method turns to two Q-tables alternately performing Q-learning updates, which has a positive bias. This feature helps to enable fine-grained control over a large bias range, from negative bias to positive bias in the large action space setting. Thus, we intuitively show that
the estimation bias of our method can be adjusted in a fine-grained and continuous manner
from negative bias to positive bias by increasing $topK$,
and there exists a suitable $K$ to remove the estimation bias.

\begin{algorithm}
	\caption{Action Intersection Double Q-learning}   
	\label{alg:ActionIntersectionDouleQlearning} 
    \raggedright
    \textbf{Init}: Q-tables $Q_0^1, Q_0^2$; start state $s_0$\\
    \textbf{Parameter}: Action Intersection set size $topK$
	\begin{algorithmic}[1] 
	
		\FOR {$t$ in $\{0,1,2,...\}$}
		\STATE Choose action $a_t$ at state $s_0$ by $\epsilon$-greedy policy
		\STATE $r_t, s_{t+1} \gets Env(s_t, a_t)$
        \STATE Sample a random probability $u \sim$ Uniform(0,1)

        \IF{$u < 0.5$}
			\STATE $\mathcal{A}^{K*}_{t+1} \gets \operatorname{arg\,topK} \limits_{a' \in \mathcal{A}}Q^1_t(s_{t+1}, a')$
            \STATE $Y_t=r_t+\gamma \max \limits_{a' \in \mathcal{A}^{K*}_{t+1}} Q^2_t(s_{t+1}, a')$
		      \STATE $Q^1_{t+1} \gets (1-\alpha)Q^1_{t} + \alpha Y_t$ 
		\ELSIF{$u >=0.5$}
			\STATE $\mathcal{A}^{K*}_{t+1} \gets \operatorname{arg\,topK} \limits_{a' \in \mathcal{A}}Q^2_t(s_{t+1}, a')$
            \STATE $Y_t=r_t+\gamma \max \limits_{a' \in \mathcal{A}^{K*}_{t+1}} Q^1_t(s_{t+1}, a')$
		      \STATE $Q^2_{t+1} \gets (1-\alpha)Q^2_{t} + \alpha Y_t$ 
		\ENDIF

		\ENDFOR
		
	\end{algorithmic}
\end{algorithm}

\section{Theoretical Analysis}

Let $W = \{W_1,..,W_M\}$ be a $M$ dimension vector of independent random variables.
We wish to estimate 
\begin{equation}
    \max_{i}\mathbb{E}[W_i].
\end{equation}
We have a sample set $S$ which is randomly sampled from $W$. Given split parameter $\rho \in [0,1]$, we can randomly split $S$ into two disjoint subsets $S^A$ and $S^B$, and $S^{C}$ with equal proportion. e.g.
\begin{equation}
    \begin{cases}
        S = S^A \cup S^B,\\
        S^A \cap S^B = \phi,\\
        |S^A|:|S^B|=1:1.
    \end{cases}
\end{equation}
Furthermore,We can interpret  $S^A$ and $S^{B}$ as collections of random variables:
\begin{equation}
    \begin{cases}
        S^A = \{X_{i,j}|i \in \{1,...,M\}, j \in \{1,...,N_A\}\},\\
        S^B = \{Y_{i,j}|i \in \{1,...,M\}, j \in \{1,...,N_B\}\},\\
    \end{cases}
\end{equation}
which satisfies:
\begin{equation}
	N_A:N_B = 1:1.
\end{equation}
For example, $X_{i,j}$ means the $j^{th}$ sample of $W_i$.

We define the following empirical mean estimator based on  $S^A$ and $S^B$, respectively:
\begin{equation}
	\bar{X_i} \stackrel{\mathrm{def}}{=} \frac{\sum_{j=1}^{N_A}X_{i,j}}{|N_A|} ,
\end{equation}

\begin{equation}
	\bar{Y_i} \stackrel{\mathrm{def}}{=} \frac{\sum_{j=1}^{N_B}Y_{i,j}}{|N_B|} ,
\end{equation}	

Thus, the optimal action set $\mathcal{A}^*$ with size $topK$ is
\begin{equation}
    \mathcal{A}^*=\operatorname{arg\,topK} \limits_{i \in \{0,1,...,M-1\}} \bar{X_i}
\end{equation}
Our estimator is
\begin{equation}
    \bar{Y}_{a^*},
\end{equation}
where
\begin{equation}
    a^*=\arg\max_{i\in \mathcal{A}^*}\bar{Y_i}.
\end{equation}

\begin{theorem}
\label{theo:unbiasedestimation}
Let $S^A$ and $S^B$ be the independent subsets of observations sampling from the random variable $W$, with an equal proportion. Let $\bar{X}_i$ and $\bar{X}_i$ be the mean estimator of $\mathbb{E}[W_i]$ based on $S^A$ and $S^B$, respectively. Let $\bar{Y}_{a*}$ be the estimator of $\max_i\mathbb{E}[W_i]$, where $a^* = \arg\max_{i \in \mathcal{A}^*}\bar{X}_i$, where $\mathcal{A}^*=\operatorname{arg\,topK} \limits_{i \in \{0,1,...,M-1\}} \bar{X_i}$ Then, there exists a $topK$ such that $\mathbb{E}[\bar{Y}_{a*}] =\max_i\mathbb{E}[W_i]$.
\end{theorem}
\begin{proof}
 The proof is provided in Section \ref{sec:proof}.
\end{proof}

\section{Action Intersection Double DQN}

We denote the two deep Q-networks as $Q(s,a;\theta^1_t)$ and $Q(s,a;\theta^2_t)$ that are parameterized by $\theta^1_t$ and $\theta^2_t$, respectively. The input of the deep neural network is state $s$ and action $a$, the output is the estimated Q-value of state-action pair $(s,a)$, i.e.,
\begin{equation}
    \begin{cases}
        \hat{Q}^1_{s,a} = Q(s,a;\theta^1_t), \\
        \hat{Q}^2_{s,a} = Q(s,a;\theta^2_t).
    \end{cases}
\end{equation}

Then we denote two target networks as $Q(s,a;\theta^{1, target}_t)$ and $Q(s,a;\theta^{2, target}_t)$ that are parameterized by $\theta^{1, target}_t$ and $\theta^{2, target}_t$, respectively. They are delayed copies of the online networks. 

At the first step $t=0$, we randomly initialize the two deep neural networks $Q(s,a;\theta^1_t)$ and $Q(s,a;\theta^2_t)$ with parameters $\theta^1_0$ and $\theta^2_0$, respectively. Then we copy $\theta^1_0$ and $\theta^2_0$ to $\theta^{1, target}_0$ and $\theta^{2, target}_0$, respectively (i.e., $\theta^{1, target}_0=\theta^1_0$ and $\theta^{2, target}_0=\theta^2_0$).

At any time step $t$, we generate a policy $\pi_t:\mathcal{S} \rightarrow \mathcal{A}$ by  
\begin{equation}
    a=\pi_t(s) \stackrel{\mathrm{def}}{=} \arg\max_{a'}\left[Q(s,a';\theta^1_t)+Q(s,a';\theta^2_t)\right].
\end{equation}

Then, at any time step $t$, the agent is at state $s_t$ and uses $\epsilon$-greedy policy with $\pi_t$ to decide the action $a_t$ to take where $\epsilon\in[0,1]$:
\begin{equation}
    a_t =
    \begin{cases}
        Sample(\mathcal{A}) &Prob.=\epsilon\\
        \pi_t(s_t) &Prob.= 1-\epsilon
    \end{cases}
\end{equation}
where $sample(\mathcal{A})$ is to choose a random action from $\mathcal{A}$ with equal probability. The environment receives the action $a_t$ and then gives the reward $r_t$ and the next state $s_{t+1}$.
This process generates a sampling data $(s_t, a_t,r_t, s_{t+1})$. The data will be added to the replay buffer $\mathcal{B}$:
\begin{equation}
    \mathcal{B}=\mathcal{B} \cup \{(s_t, a_t,r_t, s_{t+1})\}
\end{equation}

The agent will repeat the above exploration process 1000 times. It makes sure that there are enough samples in the replay buffer to update networks in the following steps. 

At time step $t>1000$, the agent first interacts with the environment to sample trajectory data and update buffer $\mathcal{B}$ using the above exploration process. Secondly, sample a mini-batch $\mathcal{B}_{mini}$ uniformly at random without replacement from the buffer $\mathcal{B}$. Then update the Q-networks with the following rules. 

Firstly, select either $Q(\cdot;\theta^1_t)$ or $Q(\cdot;\theta^2_t)$ with equal probability to update (e.g., $Q(\cdot;\theta^1_t)$). For any sample $(s_i,a_i,r_i,s_{i+1}$ in $\mathcal{B}_{mini}$, we select the optimal action intersection set $ \mathcal{A}_{i+1}^{K*}$ for state $s_{i+1}$:
\begin{equation}
    \mathcal{A}_{i+1}^{K*}= \operatorname{arg\,topK} \limits_{a' \in \mathcal{A}}Q(s_{i+1}, a';\theta^1_t).
\end{equation}

The target value for $Q(s_i,a_i;\theta^1_t)$ is
\begin{equation}
    Y_i(s_i,a_i) = r_i +\gamma \max\limits_{a' \in \mathcal{A}^{K*}_{i+1}} Q(s_{i+1},a';\theta^{2, target}_t).
\end{equation}

The loss function is 
\begin{equation}
    J(\theta^1_t)=\frac{1}{|\mathcal{B}_{mini}|}\sum_{(s_i,a_r,r_i,s_{i+1})\in\mathcal{B}_{mini}} \left[Y_i(s_i,a_i)- Q(s_i,a_i;\theta^1_t)\right]^2.
\end{equation}

update parameters $\theta_{t}^1$ as:

\begin{equation}
    \theta_{t+1}^1 = \theta_{t}^1 + \alpha\frac{\partial J(theta_{t}^1)}{\partial\theta_{t}^1},
\end{equation}
where $\alpha$ is the learning rate.

If $Q(\cdot;\theta^2_t)$ is chosen to be updated, then exchange the role of the two Q-networks in the update rule of $Q(\cdot;\theta^1_t)$:

\begin{equation}
    \begin{cases}
        \mathcal{A}_{i+1}^{K*}= \operatorname{arg\,topK} \limits_{a' \in \mathcal{A}}Q(s_{i+1}, a';\theta^2_t), \\
        Y_i(s_i,a_i) = r_i +\gamma \max\limits_{a' \in \mathcal{A}^{K*}_{i+1}} Q(s_{i+1},a';\theta^{1, target}_t), \\
         J(\theta^2_t)=\frac{1}{|\mathcal{B}_{mini}|}\sum_{(s_i,a_r,r_i,s_{i+1})\in\mathcal{B}_{mini}} \left[Y_i(s_i,a_i)- Q(s_i,a_i;\theta^2_t)\right]^2, \\
        \theta_{t+1}^2 = \theta_{t}^2 + \alpha\frac{\partial J(\theta^2_t)}{\partial\theta_{t}^2}.
    \end{cases}
\end{equation}

Periodically, the parameters of the online network are copied to the target network, e.g., if $t$ mod $a==0$, then

\begin{equation}
    \begin{cases}
        \theta^{1, target}_{t+1}=\theta^1_{t+1}, \\
        \theta^{2, target}_{t+1}=\theta^2_{t+1},
    \end{cases}
\end{equation}
where a is the interval of copy.

Algorithm \ref{alg:ActionIntersectionDDQN} extends our AIDQ to complex and high-dimensional domains,
leads to \underline{A}ction \underline{I}ntersection \underline{D}ouble \underline{DQN} (AIDDQN).
We first parameterize the two Q-functions with deep neural networks $Q(s,a;\theta^1_t)$ and $Q(s,a;\theta^2_t)$.
And then to ensure stable learning in deep reinforcement learning setting \cite{mnih2013playing,mnih2015human},
we use two target networks denoted as $Q(s,a;\theta^{1, target}_t)$ and $Q(s,a;\theta^{2, target}_t)$, 
which are delayed copies of the online networks. 
These target networks are used to compute the temporal-difference targets, 
thereby preventing the feedback loops that can lead to divergence.

\begin{algorithm}
	\caption{Action Intersection DDQN}   
	\label{alg:ActionIntersectionDDQN} 
    \raggedright
    \textbf{Init:} \hspace{0.1em}\parbox[t]{0.9\linewidth}{
        Four Q-networks with random parameters $\theta^1_0$, $\theta^2_0$, $\theta^{1,\text{target}}_0$, $\theta^{2,\text{target}}_0$; replay buffer $\mathcal{B}=\{\}$; start state $s_0$}
    \textbf{Parameter}: Action Intersection set size $topK$; target Q-networks update step $V$
	\begin{algorithmic}[1] 
		\FOR {$t$ in $\{0,1,2,...\}$}

        \IF{$t$ mod $V$ ==0}
            \STATE $\theta^{1,\text{target}}_t \gets \theta^1_t$
            \STATE $\theta^{2,\text{target}}_t \gets \theta^2_t$
            
        \ENDIF
        
		\STATE Choose action $a_t$ at state $s_t$ by $\epsilon$-greedy policy
		\STATE $r_t, s_{t+1} \gets Env(s_t, a_t)$
		
		\STATE $\mathcal{B} \gets \mathcal{B} \cup \{(s_t, a_t, r_t, s_{t+1})\}$ \COMMENT{Update buffer}
		
		\STATE Sample minibatch $\mathcal{B}_{mini}$ of transitions from $\mathcal{B}$
		\STATE Sample a random probability $u \sim \mathrm{Uniform}(0,1)$
		
		\IF{$u<0.5$} 
        \FOR {$(s_i,a_i,r_i,s_{i+1}) \in \mathcal{B}_{mini}$}
            \STATE$ \mathcal{A}^{K*}_{i} \gets \operatorname{arg\,topK}\limits_{a' \in \mathcal{A}}Q(s_{i+1}, a';\theta^{1,target}_{t})$
		      \STATE $Y_i = r_i+\gamma \max\limits_{a' \in \mathcal{A}^{K*}_{i}} Q_{t}(s_{i+1},a';\theta^{2, target}_t)$
        \ENDFOR
    
		\STATE minimize $  \mathbb{E}_{(s_i,a_i,r_i,s_{i+1}) \in \mathcal{B}_{mini}}\left[Y_i - Q(s_i, a_i;\theta^1_t)\right]^2$ 
		
		\ELSE
		\FOR {$(s_i,a_i,r_i,s_{i+1}) \in \mathcal{B}_{mini}$}
            \STATE$ \mathcal{A}^{K*}_{i} \gets \operatorname{arg\,topK}\limits_{a' \in \mathcal{A}}Q(s_{i+1}, a';\theta^{2,target}_{t})$
		      \STATE $Y_i = r_i+\gamma \max\limits_{a' \in \mathcal{A}^{K*}_{i}} Q(s_{i+1},a';\theta^{1, target}_t)$
        \ENDFOR
		\STATE minimize $  \mathbb{E}_{(s_i,a_i,r_i,s_{i+1}) \in \mathcal{B}_{mini}}\left[Y_i - Q(s_i, a_i;\theta^2_t)\right]^2$

		\ENDIF

		\ENDFOR
		
	\end{algorithmic}
\end{algorithm}

\section{Experiment}
In this section, we evaluate the proposed method on both tabular settings and deep network settings. In the tabular setting, we evaluate and compare our method with 8 baselines in a Multi-armed bandit. In the deep network setting, we evaluate and compare our method with 8 baselines on 6 simulated environments. The code is at anonymous \url{https://anonymous.4open.science/r/AIDQ-D826/Readme.txt}.

\subsection{Multi-Armed Bandit}
\label{sec:exp-bandit}
\textbf{Environment:} 
Following previous work \cite{zhang2017weighted}, we consider a simple multi-armed bandit setting.
It includes one unterminated state $S$ with $N$ arms, i.e., $\mathcal{A}=\{a_i|i=1,2,..,N\}$.
For any arm $a_i$, it yields a reward drawn from a normal distribution $\mathcal{N}(\mu_i, \sigma^{2}_R)$, 
where $\mu_i=0$.

\noindent
\textbf{Parameters:}
Following previous work \cite{zhang2017weighted},
we set the discount factor $\gamma=0.95$,
the learning rate $\alpha_t(S,a)=\frac{1}{n_t(S,a)^{0.8}}$,
the exploration parameter $\epsilon_t(S)=\frac{1}{n_t(S)^{0.5}}$,
where $n_t(S)$ and $n_t(S,a)$ are the visited number of state $S$ and state-action pair $(S,a)$, respectively.
Note that under the above setting,
the optimal policy always chooses the action $a_1$ at state $S$,
and the maximum expected Q-value is as 
$\lim_{T \to +\infty}\sum_{t=0}^{T} \gamma^{t}=0$.
The initial Q-values obey $\mathcal{N}(0, \sigma^2_Q)$,
and all results are averaged over $1000$ runs.

\noindent
\textbf{Baseline:} 
We consider eight SOTA baselines as follows:
\begin{itemize}
    \item \textbf{Q-learning} \cite{watkins1992q} is one of the most classical algorithms of reinforcement learning. However, it faces the challenge of the overestimation bias.
    \item \textbf{Double Q-learning} \cite{hasselt2010double} was proposed to tackle the overestimation issue in standard Q-learning. It decouples the action selection from the value evaluation by maintaining two separate Q-value estimators. While effectively reducing overestimation, this decoupling can introduce an underestimation bias, potentially slowing down convergence.

    \item \textbf{Weighted Double Q-learning} \cite{zhang2017weighted} introduces a weighted averaging scheme to balance between the two Q-estimators in Double Q-learning. By dynamically adjusting the weights based on the uncertainty of each estimator, it aims to strike a better bias-variance trade-off, mitigating both overestimation and severe underestimation.

    \item \textbf{Averaged Q-learning} \cite{anschel2017averaged} proposes a simple yet effective approach that averages multiple Q-value estimates during the update. This averaging operation naturally reduces the variance of the value estimates, leading to more stable learning, though it may sometimes oversmooth the value function.

    \item \textbf{Maxmin Q-learning} \cite{lan2020maxmin} adopts a conservative strategy by using the minimum value among an ensemble of Q-functions for the target update. This design robustly controls overestimation and has shown strong performance in environments with high stochasticity, but may become overly pessimistic in deterministic settings.

    \item \textbf{Ensemble Bootstrapped Q-learning (EBQL)} \cite{peer2021ensemble} leverages an ensemble of Q-estimators, each trained on a different bootstrapped sample of the experience replay buffer. The diversity within the ensemble helps in better exploration and more robust value estimation, though at the cost of increased computational complexity.

    \item \textbf{Order Q-learning} \cite{tan2024oder} is a recent approach that employs order statistics to combine multiple Q-value predictions. By selecting a specific percentile (e.g., median or a higher-order statistic) as the target value, it attempts to interpolate between overestimation and underestimation, offering a flexible framework for bias control.
    
    \item \textbf{AC-C Double Q-learning (AC-CDQ)} \cite{jiang2021action} proposes a clipped double estimator for Double Q-learning, in order to reduce the underestimation bias. It maintains two estimators. When calculating the temporal difference target, the corresponding action value of the selected action in the first set of estimators is clipped by the maximum value in the second set of estimators.
    
\end{itemize}
 
To ensure fairness, all algorithms maintain two Q-tables, except for Q-learning. 
For the hyperparameters of algorithms, Weighted Double Q-learning \cite{zhang2017weighted} sets the adaptive adjustment factor as $10$;
Order Q-learning \cite{tan2024adaptive} sets the index of the order function as $2$.
AC-CDQ sets the candidate size as $2$.
Note that these hyperparameters of comparison baselines are chosen from the recommended hyperparameters,
which are fine-tuned.

We consider three different settings to evaluate the effect of our data multiplexing strategy as follows:
\begin{itemize}
    \item Case 1: different arms with $N \in \{20,40,60,80\}$;
    \item Case 2: different reward variance with $\sigma_R \in \{5,10,15,20\}$;
    \item Case 3: different initial Q-values variance with $\sigma_Q \in \{1,2,4,8\}$.

\end{itemize}
We set the arms with $N=40$, the reward variance with $\sigma_R=10$,
and the initial Q-values variance with $\sigma_Q=1$ for default.

\begin{figure}
  \centering

  \includegraphics[width=\linewidth]{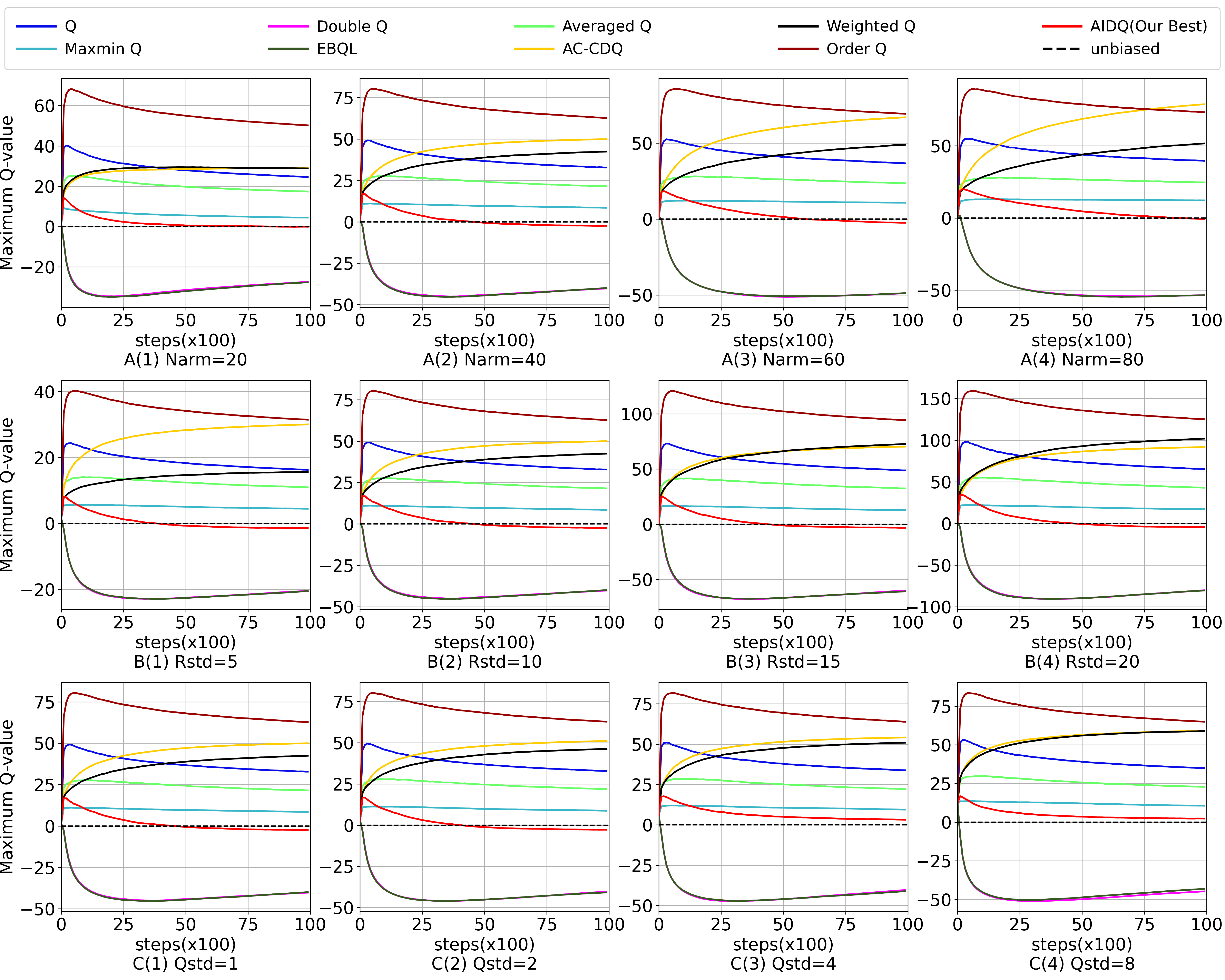}
  \caption{The average return per episode of eight different algorithms in comparison.}
  \label{fig:tablebaselines}
\end{figure}

\textbf{Baseline Comparison} 
Figure \ref{fig:tablebaselines} shows all experiment result of all baselines under 3 cases. For any subplot in Figure \ref{fig:tablebaselines}, the x-axis represents the update steps, and the y-axis represents the maximum Q-value estimated by the algorithm for state $S$. The closer an algorithm's curve is to the dashed unbiased estimation curve, the more accurate its Q-value estimation for state $S$.

Figure \ref{fig:tablebaselines}A(1)-A(4) shows the experiment result under case 1, where the number of arms is varied. There are three observations. 
(1) under different cases, the red curve of our method is in the middle and most closely aligns with the dashed unbiased estimation line. This implies that our method is able to achieve unbiased estimation with a large action space. Specifically, after 10,000 training steps, the maximum Q-value estimated by our algorithm aligns almost perfectly with the true Q-value.
(2) Under different cases, the curve of the baseline method deviates from the dashed unbiased estimation line. This implies that baselines suffer from estimation bias under the large space action setting. More specifically, Double Q and EBQL have a negative bias; other methods have a positive bias.
(3) As the number of arms increases, the deviation between the curve of the baseline algorithm and the unbiased estimation curve widens; the number of training steps required for our algorithm's curve to begin converging toward the unbiased estimation curve increases. 

Figure \ref{fig:tablebaselines}B(1)-B(4) shows the experiment result under case 2, where the standard deviation of reward are varied. There are two observations:
(1) under different cases, the red curve of our method is in the middle and most closely aligns with the dashed unbiased estimation line. This implies that the action intersection strategy is able to achieve unbiased estimation.
(2) As the standard deviation of reward increases, The distance between the curve of the baseline algorithm and the unbiased estimation curve increases significantly(e.g., when $Rstd=5$, the maximum Q value of Order Q can be up to 40; when $Rstd=20$, the maximum Q value of Order Q can be up to 150);  Our algorithm converges to the unbiased estimator in a similar number of steps (25,000) under all four conditions. This implies that our action intersection strategy is less affected by reward variance.

Figure \ref{fig:tablebaselines}C(1)-C(4) shows the experiment result under case 3, where the standard deviation of the initialized Q-value is varied. There are two observations:
(1) under different cases, the red curve of our method is in the middle and most closely aligns with the dashed unbiased estimation line.
(2) As the standard deviation of initialized Q-values increases, all curves of our method and baselines maintain the same relative positions. This implies that the variance in the initialized Q-values has little impact on the algorithm's performance.

\begin{figure}
  \centering

  \includegraphics[width=\linewidth]{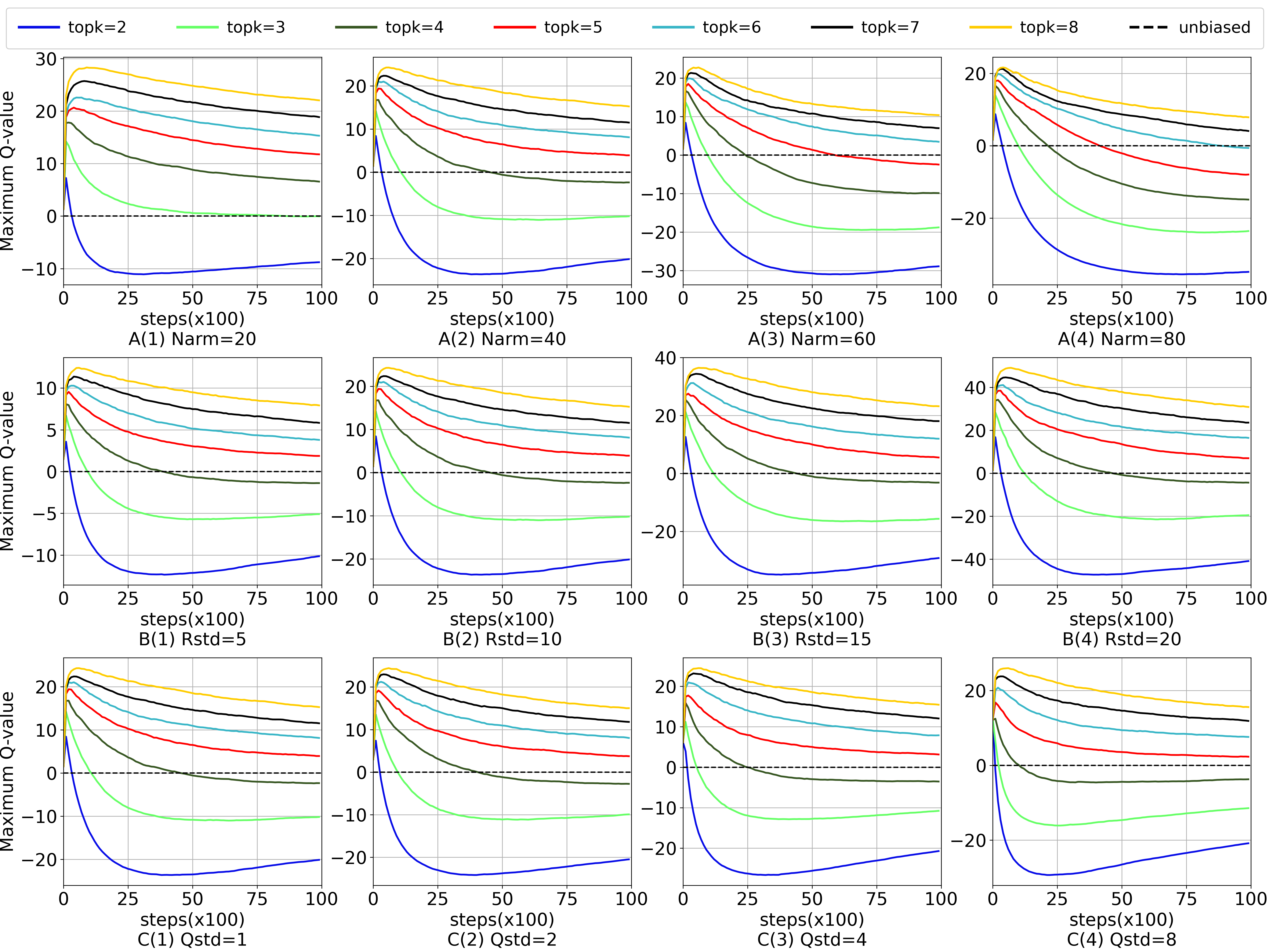}
  \caption{The maximum Q-value of our AIDQ with different $topK$.}
  \label{fig:tablepara}
\end{figure}

\textbf{Parameter Sensitivity} 
Figure \ref{fig:tablepara} shows the impact of the hyperparameter $topK$ on our algorithm under 3 cases. For any subplot in Figure \ref{fig:tablepara}, the x-axis represents the update steps, and the y-axis represents the maximum Q-value estimated by the algorithm for state $S$. The closer an algorithm's curve is to the dashed unbiased estimation curve, the more accurate its Q-value estimation for state $S$.

Figure \ref{fig:tablepara}A(1)-A(4) shows the experimental results on the sensitivity of parameter $topK$ in our algorithm under different action space sizes. There are three observations:
(1) Under different settings for the number of actions, as the value of $topK$ increases, the estimation bias can vary from negative bias to positive bias. This implies that it is possible to find a suitable value for $topK$ such that the estimated maximum Q-value approximates the unbiased Q-value.
(2) Under different settings for the number of actions, our algorithm can always find a suitable $topK$ such that the estimated maximum Q-value approximates the unbiased Q-value. This implies that it is feasible for our algorithm to find a suitable parameter $top-k$ during deployment.
(3) As the number of arms increases, the suitable value of $topK$ increases. When the number of arms is 20, the suitable $topK$ is 3; when the number of arms is 80, the suitable value of $topK$ is 6. The optimal value of parameter $topK$ increases approximately linearly with the number of actions. This implies that it aligns well with empirical intuition regarding our method’s hyperparameter tuning.

\begin{sidewaystable}
\centering
\caption{The maximum Q-value of different algorithms under three different settings. The optimal estimation values in each column are highlighted in bold. The second best result in each column are underlined.}
\begin{tabular}{lcccccccccccc}
\toprule 
\multirow{2.5}{*}{Algorithm} & \multicolumn{4}{c}{$N_{arm}$} & \multicolumn{4}{c}{$\sigma_{R}$} & \multicolumn{4}{c}{$\sigma_{Q}$} \\
 \cmidrule(lr){2-5} \cmidrule(lr){6-9} \cmidrule(lr){10-13}
 & 20 & 40 & 60 & 80 & 5 & 10 & 15 & 20 & 1 & 2 & 4 & 8 \\
\midrule 
Q & 24.64 & 32.84 & 36.72 & 39.60 & 16.30 & 32.84 & 48.81 & 65.30 & 32.84 & 32.89 & 33.80 & 34.96 \\
Double Q & -27.37 & -40.23 & -48.87 & -53.48 & -20.44 & -40.23 & -59.96 & -80.29 & -40.23 & -40.33 & -40.40 & -44.75 \\
Weighted Q & 28.86 & 42.46 & 48.92 & 51.60 & 15.65 & 42.46 & 72.69 & 101.75 & 42.46 & 46.33 & 50.89 & 58.87 \\
Averaged Q & 17.37 & 21.46 & 23.58 & 24.64 & 10.98 & 21.46 & 32.49 & 42.81 & 21.46 & 21.89 & 22.15 & 22.83 \\
Maxmin Q & \underline{4.47} & 8.51 & 10.80 & 12.19 & 4.48 & 8.51 & 12.80 & 17.09 & 8.51 & 8.78 & 9.46 & 10.64 \\
EBQL & -27.61 & -39.95 & -48.73 & -53.52 & -20.51 & -39.95 & -60.75 & -80.52 & -39.95 & -40.79 & -41.08 & -43.14 \\
AC-CDQ & 29.11 & 49.98 & 67.01 & 78.79 & 30.08 & 49.98 & 70.34 & 91.65 & 49.98 & 51.14 & 54.15 & 59.05 \\
Order Q & 50.22 & 62.83 & 69.31 & 73.29 & 31.50 & 62.83 & 94.33 & 125.05 & 62.83 & 62.84 & 63.84 & 64.93 \\
\midrule
AIDQ(topK=2) & -8.79 & -20.14 & -28.87 & -34.90 & -10.14 & -20.14 & -29.22 & -40.92 & -20.14 & -20.48 & -20.72 & -20.85 \\
AIDQ(topK=3) & \textbf{-0.06} & -10.20 & -18.76 & -23.61 & -5.10 & -10.20 & -15.67 & -19.55 & -10.20 & -9.91 & -10.78 & -11.45 \\
AIDQ(topK=4) & 6.56 & \textbf{-2.35} & -9.86 & -14.89 & \textbf{-1.38} & \textbf{-2.35} & \textbf{-3.17} & \textbf{-4.38} & \textbf{-2.35} & \textbf{-2.75} & \underline{-3.53} & \underline{-3.72} \\
AIDQ(topK=5) & 11.75 & \underline{3.93} & \textbf{-2.45} & -7.96 & \underline{1.87} & \underline{3.93} & \underline{5.50} & \underline{6.98} & \underline{3.93} & \underline{3.74} & \textbf{3.15} & \textbf{2.32} \\
AIDQ(topK=6) & 15.31 & 8.12 & \underline{3.46} & \textbf{-0.62} & 3.79 & 8.12 & 11.98 & 16.46 & 8.12 & 8.03 & 7.89 & 7.61 \\
AIDQ(topK=7) & 18.84 & 11.52 & 6.99 & \underline{4.09} & 5.83 & 11.52 & 18.05 & 23.57 & 11.52 & 11.80 & 12.05 & 11.89 \\
AIDQ(topK=8) & 22.03 & 15.27 & 10.37 & 7.88 & 7.92 & 15.27 & 23.19 & 30.85 & 15.27 & 15.02 & 15.45 & 15.60 \\
\bottomrule 
\end{tabular}
  \label{tab:tableres}
\end{sidewaystable}

Figure \ref{fig:tablepara}B(1)-B(4) shows the experimental results on the sensitivity of parameter $topK$ in our algorithm under different standard deviations of reward. One can observe that
(1) Under different settings for the standard deviation, as the value of $topK$ increases, the estimation bias can vary from negative bias to positive bias. This implies that it is possible to find a suitable value for $topK$ such that the estimated maximum Q-value approximates the unbiased Q-value.
(2) Under different settings for the standard deviation, our algorithm can always find a suitable $topK$, which is 5, such that the estimated maximum Q-value approximates the unbiased Q-value. This implies that the parameter $topK$ is insensitive to reward variance.

Figure \ref{fig:tablepara}C(1)-BC(4) shows the experimental results on the sensitivity of parameter $topK$ in our algorithm under different standard deviations of initialized Q-values. One can observe that
(1) Under different settings for the standard deviation, as the value of $topK$ increases, the estimation bias can vary from negative bias to positive bias. This implies that it is possible to find a suitable value for $topK$ such that the estimated maximum Q-value approximates the unbiased Q-value.
(2) As the standard deviation of initialized Q-values increases, the positions of the curves for different values of parameter K show little variation. Our algorithm can always find a suitable $topK$, which is 4, such that the estimated maximum Q-value approximates the unbiased Q-value. This implies that the parameter $topK$ is insensitive to the variance of initialized Q-values.

Table \ref{tab:tableres} comprehensively shows the maximum Q-value of different algorithms under three different settings,
where the values are evaluated in the final learning steps.
One can observe that across all experimental configurations (each column), our algorithm yields estimates that are closest to the ground truth value of 20.

\begin{figure}
  \centering
  \includegraphics[width=\linewidth]{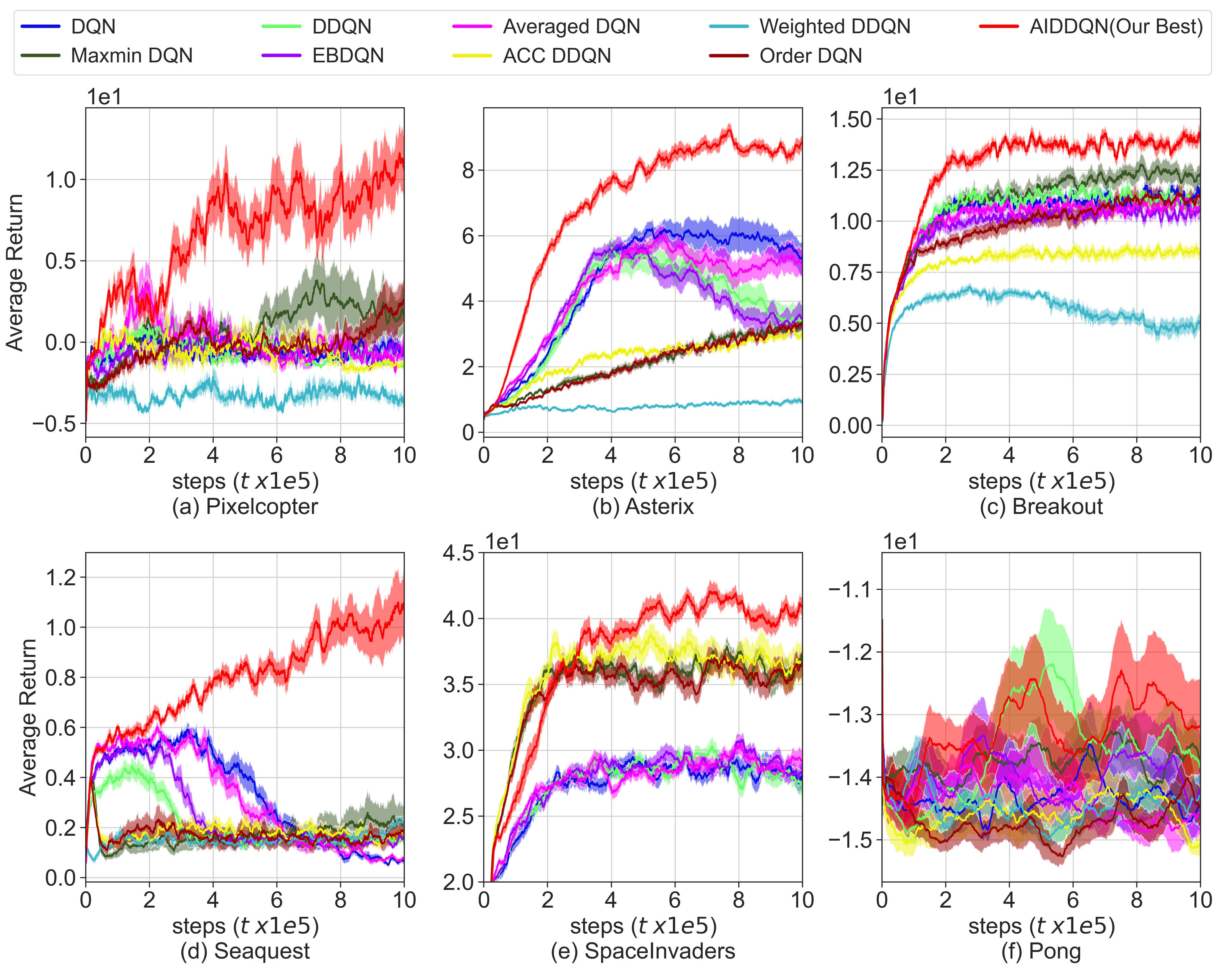}
  \caption{The average return per episode of eight different algorithms in comparison.}
  \label{fig:deepbaselines}
\end{figure}

\subsection{Deep Reinforcement Learning}
\textbf{Environment:} 
Following previous work \cite{lan2020maxmin}, this paper selects Pixelcopter, Asterix, Breakout, Seaquest, SpaceInvaders, and Pong from Gymnasium \cite{towers2024gymnasium}, PyGame Learning Environment (PLE) \cite{tasfi2016PLE}, and MinAtar\cite{gym-games} as our deep reinforcement learning environments.
For example, Seaquest environment is one of Atari 2600 games, but a simplified one to make experimentation more accessible and efficient. The default observation is a visual input of the game. Thus, the state space is complex, and it is hard for the agent to capture the relationship among states, actions, and rewards. 

\noindent
\textbf{Parameter:} 
Following previous work \cite{lan2020maxmin},
we set the discount factor $\gamma=0.99$,
the Adam optimizer \cite{zhang2018improved} with the learning rate $\alpha=0.0005, 0.001, 0.001, 0.001, 0.001, 0.0009, 0.0001$ for environments Pixelcopter, Asterix, Breakout, Seaquest, SpaceInvaders, and Pong, respectively
the mini-batch size $\vert\mathcal{B}_{mini}\vert=32$,
the replay buffer size  $\vert\mathcal{B}\vert=100,000$,
the update steps of the target Q-network $200$.
Note that $\epsilon$-greedy is applied as the exploration strategy with 
decreasing linearly from $1.0$ to $0.01$ in $1,000$ steps,
and after $1,000$ steps, $\epsilon$ is fixed to $0.01$.
All results reported are the average and standard deviation over $20$ independent runs 
with different random seeds. Note that the original action spaces of these environments were not large. To simulate a large action-space setting, we multiplied their action spaces by a factor. The scaling factors for environments Pixelcopter, Asterix, Breakout, Seaquest, SpaceInvaders, and Pong are 40, 20, 20, 20, 20, and 20, respectively.

\noindent
\textbf{Baseline:} 
To ensure consistency, the configurations of baseline algorithms are the same 
with the previous experiments in Section \ref{sec:exp-bandit}.

\begin{figure}
  \centering
  \includegraphics[width=\linewidth]{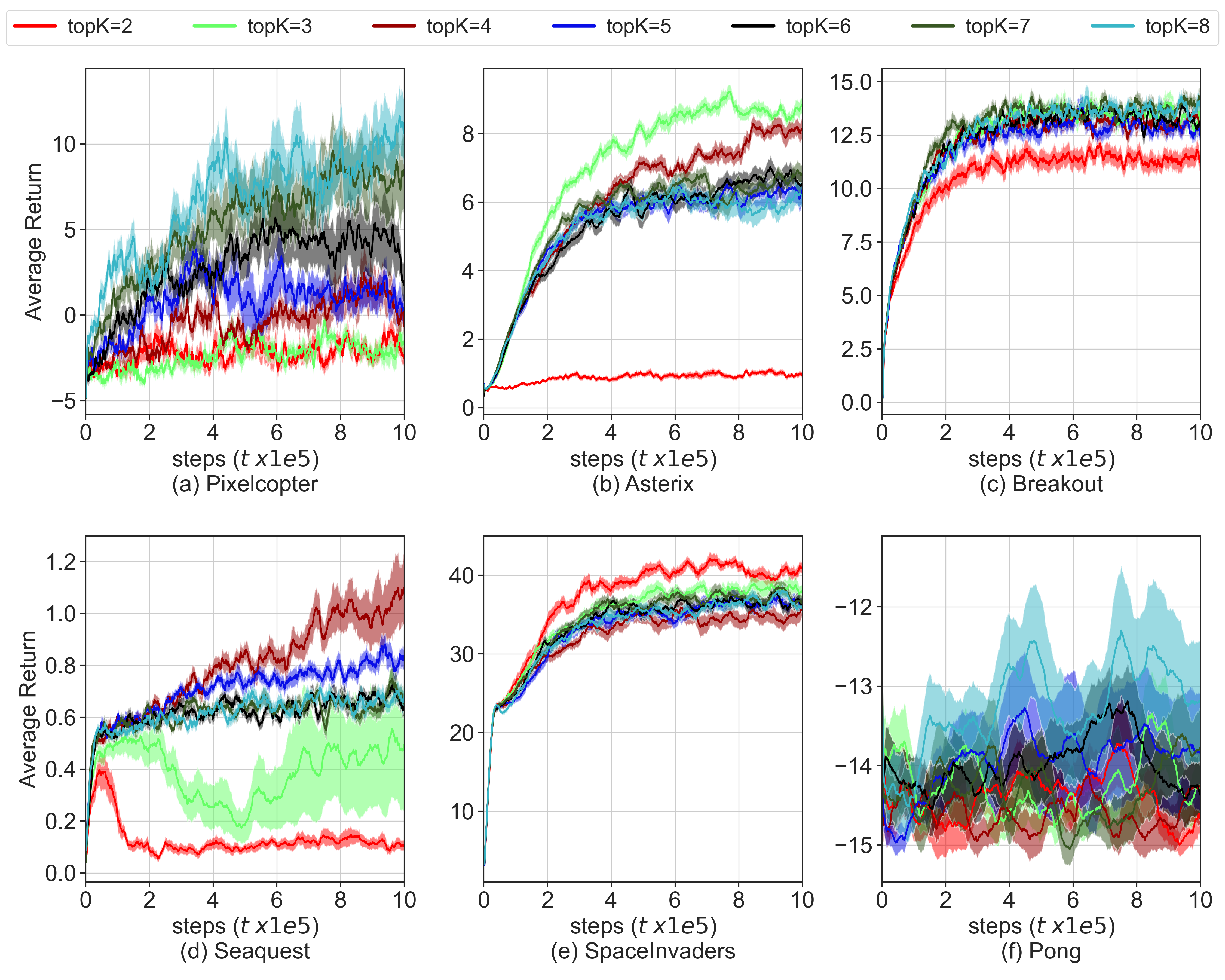}
  \caption{The average return per episode of our AIDQ with different $topK$.}
  \label{fig:deeppara}
\end{figure}

\textbf{Baseline Comparison} 
Figure \ref{fig:deepbaselines} shows the average return per episode of nine different algorithms in comparison across step $t$, where our AIDDQN varies the $topK=2,3,4,5,6,7,8$ and chooses the best result. One can observe that:
in all environments, the red reward curve of our method is above all the baseline curves. This implies that in the setting of a large action space, our method can achieve a better performance than other baselines. This is beneficial from the action intersection strategy, which is aiming at a large action space setting. More specifically, 
in (a) Pixecopter, the red curve of our method is at the top; the curve of Weighted DQN is at the bottom; the curves of other baselines are similar. This implies that in a large action space, all baselines have their bottleneck. 
In (b) Asterix, the curve of Weighted DQN is at the bottom; the curves of Maxmin DQN, ACC DDQN, and Order DQN are similar; the curves of other baselines are similar, which are at the middle. This implies that in this environment,  
In (c) Breakout, all curves are well-separated from each other.
In (d) Seaquest, all baseline curves show low performance levels. Several algorithms exhibit a performance trend that declines after peaks. Note that, in the parameter sensitivity experiment, when $topK$ takes certain values, the corresponding curves of our algorithm also exhibit a similar pattern. 
In (e) SpaceInvaders, the curves of Maxmin DQN, ACC DDQN, and Order DQN are similar, which are in the middle; the curves of other baselines are similar, which are at the bottom.
In (d) Pong, all curves exhibit considerable fluctuations and high variance. Nevertheless, the curve of our method is above all the baseline curves for most of the time.

\textbf{Parameter Sensitivity} 
Figure \ref{fig:deeppara} shows the average return per episode of our AIDDQN with different $topK$ across step $t$.
In (a) Pixelcopter, as the value of $topK$ increases, the corresponding curve gradually rises in position. This implies that the performance of our algorithm exhibits a clear positive correlation with the value of parameter $topk$ in this environment.
In (b) Asterix and (d) Seaquest, as the value of $topK$ increases, the position of the corresponding curve first increases and then decreases, with the optimal value achieved at $topK=3$. 
In (c) Breakout, as the value of $topK$ increases, a noticeable shift in the corresponding curve occurs when $topK$ changes from 2 to 3, while subsequent changes in $topK$ result in minimal variation in the curve's position. 
In (d) SpaceInvaders, as the value of $topK$ increases, the position of the corresponding curve first decreases and then increases, with the optimal value achieved at $topK=2$. 
In (e) Pong, all curves exhibit considerable fluctuations and high variance. Nevertheless, the optimal parameter is identified as $topK=8$.

Table \ref{tab:deeptableres} shows the average return per episode of different SOTA algorithms 
and our AIDDQN with different $topK$,
where the values represent the mean of the reward data from the final $10\%$ of the training process. 
One can observe that
our AIDDQN improves the average return per episode over SOTA baselines by at least
$(10.40-1.85)/1.85=463.45\%$, 
$(8.60-5.61)/5.61=53.21\%$,
$(13.90-12.23)/12.23=13.66\%$,
$(1.03-0.22)/0.22=366.95\%$, 
$(40.26-36.38)/36.38=10.65\%$, 
$((-13.06)-(-13.54))/|-13.54|=3.55\%$, 
in Pixelcopter, Asterix, Breakout, Seaquest, SpaceInvaders, and Pong, respectively.

\begin{sidewaystable}
\centering
\caption{The average return per episode of different algorithms under different environments. The optimal averaged return in each column are highlighted in bold. The second best result in each column are underlined}
\footnotesize
\setlength{\tabcolsep}{5pt}
\begin{tabular}{lcccccccccccc}
\toprule
\multirow{2.5}{*}{Algorithm} & \multicolumn{2}{c}{Pixelcopter} & \multicolumn{2}{c}{Asterix} & \multicolumn{2}{c}{Breakout} & \multicolumn{2}{c}{Seaquest} & \multicolumn{2}{c}{SpaceInvaders} & \multicolumn{2}{c}{Pong} \\
 \cmidrule(lr){2-3} \cmidrule(lr){4-5} \cmidrule(lr){6-7} \cmidrule(lr){8-9} \cmidrule(lr){10-11} \cmidrule(lr){12-13}
 & Reward & Improv & Reward & Improv & Reward & Improv & Reward & Improv & Reward & Improv & Reward & Improv \\
\midrule
DQN & -0.31 & -117.00\% & 5.61 & 0.00\% & 11.24 & -8.08\% & 0.07 & -68.16\% & 28.36 & -22.06\% & -14.51 & -7.22\% \\
DDQN & -1.28 & -169.33\% & 3.47 & -38.20\% & 10.95 & -10.43\% & 0.15 & -30.99\% & 28.02 & -22.98\% & \underline{-13.54} & \underline{0.00}\% \\
Averaged DQN & -0.91 & -149.33\% & 5.25 & -6.47\% & 10.80 & -11.67\% & 0.07 & -66.52\% & 28.94 & -20.45\% & -14.71 & -8.67\% \\
Maxmin DQN & 1.73 & -6.20\% & 3.16 & -43.66\% & 12.23 & 0.00\% & 0.22 & 0.00\% & 35.89 & -1.37\% & -13.60 & -0.45\% \\
Weighted DQN & -3.34 & -280.85\% & 0.92 & -83.60\% & 4.82 & -60.56\% & 0.19 & -13.14\% & 6.87 & -81.13\% & -14.40 & -6.34\% \\
EBDQN & -0.59 & -131.86\% & 3.42 & -39.02\% & 10.33 & -15.51\% & 0.15 & -32.93\% & 28.65 & -21.26\% & -14.10 & -4.13\% \\
ACC DDQN & -1.34 & -172.65\% & 2.98 & -47.00\% & 8.52 & -30.36\% & 0.16 & -28.66\% & 36.38 & 0.00\% & -14.85 & -9.71\% \\
Order DQN & 1.85 & 0.00\% & 3.17 & -43.49\% & 11.06 & -9.53\% & 0.17 & -24.79\% & 35.88 & -1.39\% & -14.28 & -5.49\% \\
\midrule
topK=2 & -1.72 & -193.42\% & 0.98 & -82.56\% & 11.41 & -6.68\% & 0.11 & -50.24\% & \textbf{40.26} & \textbf{10.65}\% & -14.85 & -9.67\% \\
topK=3 & -2.09 & -213.25\% & \textbf{8.60} & \textbf{53.21}\% & 13.30 & 8.79\% & 0.48 & 119.58\% & \underline{38.34} & \underline{5.37}\% & -14.23 & -5.11\% \\
topK=4 & 0.93 & -49.47\% & \underline{8.11} & \underline{44.47}\% & 13.12 & 7.26\% & \textbf{1.03} & \textbf{366.95}\% & 34.72 & -4.56\% & -14.49 & -7.01\% \\
topK=5 & 1.11 & -40.04\% & 6.32 & 12.52\% & 12.79 & 4.60\% & \underline{0.83} & \underline{273.93}\% & 36.65 & 0.72\% & -13.83 & -2.16\% \\
topK=6 & 4.25 & 129.91\% & 6.70 & 19.37\% & 13.26 & 8.46\% & 0.68 & 206.04\% & 36.73 & 0.96\% & -14.30 & -5.66\% \\
topK=7 & \underline{8.05} & \underline{335.75}\% & 6.71 & 19.56\% & \textbf{13.90} & \textbf{13.66}\% & 0.68 & 206.46\% & 37.48 & 3.02\% & -13.76 & -1.62\% \\
topK=8 & \textbf{10.40} & \textbf{463.45}\% & 5.99 & 6.68\% & \underline{13.72} & \underline{12.17}\% & 0.67 & 203.93\% & 37.05 & 1.83\% & \textbf{-13.06} & \textbf{3.55}\% \\
\bottomrule 
\end{tabular}
  \label{tab:deeptableres}
\end{sidewaystable}

\section{Proof of Theorem \ref{theo:unbiasedestimation}}
\label{sec:proof}
\begin{proof}

We define the optimal action selected from the action space $\mathcal{A}$ by $\bar{Y_i}$ as:
\begin{equation}
    a^*_Y=\arg\max_{i\in \mathcal{A}}\bar{Y_i}.
\end{equation}

The probability that $a^*_Y$ equals $a^*$ is a function of $topK$:
\begin{equation}
    Pr(a^*_Y=a^*) = f(topK)
\end{equation}

When $a^*_Y=a^*$, our estimator is Q Estimator:
\begin{equation}
     \bar{Y}_{a*}=\bar{Y}_{a^*_Y}=\max_{i}\bar{Y}_i.
\end{equation}

When $a^*_Y=a^*$, our estimator is Double Q Estimator:
\begin{equation}
     \bar{Y}_{a*}=\bar{Y}_{a^*_{Y^-}},
\end{equation}
where $a^*_{Y^-}$ means an action which is independent of $\bar{Y}_i$.

The expectation of our estimator is
\begin{equation}
    \mathbb{E}[\bar{Y}_{a^*}]=Pr(a^*_Y=a^*)\underbrace{\mathbb{E}[\bar{Y}_{a_Y^*}]}_{term1}  + [(1-Pr(a^*_Y=a^*)]\underbrace{\mathbb{E}[\bar{Y}_{a^*_{Y^-}}]}_{term2}.
\end{equation}	

For term 1, it is a Single estimator, thus we have
\begin{equation}
	\begin{split}
        \mathbb{E}[\bar{Y}_{a_Y^*}] = \mathbb{E}[\max_{i}\bar{Y}_i]\approx \max_{i}\mathbb{E}[W_i] +n,
	\end{split}
\end{equation}	
where $n > 0$.

For term 2, it is a Double estimator, thus we have
\begin{equation}
	\begin{split}
		\mathbb{E}[\bar{Y}_{a^*_{Y^-}}] \approx \max_{i}\mathbb{E}[W_i]-m,
	\end{split}
\end{equation}	
where $m > 0 $.

Therefor,we can estimate the expectation of our estimator by
\begin{equation}
	\begin{split}
		\mathbb{E}[\bar{Y}_{a^*}] &\approx Pr(a^*_Y=a^*)(\max_i\mathbb{E}[W_i]-m)+[1- Pr(a^*_Y=a^*)](\max_i\mathbb{E}[W_i]+n) \\
		&=\max_i\mathbb{E}[W_i] + [1-Pr(a^*_Y=a^*)]n -Pr(a^*_Y=a^*)m .
    \end{split}
\end{equation}	

Let
\begin{equation}
	\begin{split}
		[1-Pr(a^*_Y=a^*)]n -Pr(a^*_Y=a^*)m = 0.
	\end{split}
\end{equation}	

We get
\begin{equation}
	\begin{split}
		Pr(a^*_Y=a^*) &= \frac{n}{m+n}.\\
        f(topK)&=\frac{n}{m+n}
	\end{split}
\end{equation}	

We can thus conclude that there exists a $topK$ such that $\mathbb{E}[\bar{Y}_{a^*}] =\max_i\mathbb{E}[W_i]$.

\end{proof}

\section{Conclusion}
This paper identifies and addresses the challenge of bias amplification in Q-learning within large action spaces. 
We propose Action Intersection Double Q-learning (AIDQ), a novel method that mitigates the estimation bias by enabling fine-grained, continuous bias control through a simple top-K action intersection mechanism. 
Our method offers control over the bias range from negative to positive. 
Extensive experiments in both tabular and deep RL settings demonstrate that AIDQ consistently outperforms state-of-the-art baselines, particularly in environments with large action spaces. 
This work opens a new direction for robust bias mitigation in value-based reinforcement learning.
Future work focuses on the adaptive parameter $topK$ tuning.

\section*{Acknowledgments}
This work is supported in part by the National Natural Science Foundation of China (Grant No. 62372427, 62476261), in part by the Strategic Priority Research Program of Chinese Academy of Sciences (Grant No. XDB1590300), in part by the Science and Technology Innovation Key R\&D Program of Chongqing (CSTB2023TIAD-STX0031, CSTB2025TIAD-STX0023), and  in part by the Natural Science Foundation of Chongqing (CSTB2025NSCQ-LZX0061).

\bibliographystyle{cas-model2-names}

\bibliography{ref}

\end{document}